\documentclass[11pt]{article}

\usepackage[margin=1.08in]{geometry}
\usepackage{amsmath,amssymb,amsthm,mathtools}
\usepackage{microtype}
\usepackage{aliascnt}
\usepackage[hidelinks]{hyperref}
\usepackage{comment}
\hypersetup{
  pdftitle={Fourth-Moment Geometry of Rademacher Sums},
  pdfsubject={Sharp Gaussian stability and finite-dimensional Khintchine inequalities}
}
\usepackage[nameinlink,capitalise,noabbrev]{cleveref}

\allowdisplaybreaks
\newtheorem{theorem}{Theorem}[section]
\newaliascnt{lemma}{theorem}
\newtheorem{lemma}[lemma]{Lemma}
\aliascntresetthe{lemma}
\newaliascnt{proposition}{theorem}
\newtheorem{proposition}[proposition]{Proposition}
\aliascntresetthe{proposition}
\newaliascnt{corollary}{theorem}
\newtheorem{corollary}[corollary]{Corollary}
\aliascntresetthe{corollary}
\newaliascnt{conjecture}{theorem}
\newtheorem{conjecture}[conjecture]{Conjecture}
\aliascntresetthe{conjecture}
\theoremstyle{definition}
\newaliascnt{definition}{theorem}

\aliascntresetthe{definition}
\theoremstyle{remark}
\newaliascnt{remark}{theorem}
\newtheorem{remark}[remark]{Remark}
\aliascntresetthe{remark}

\crefname{theorem}{Theorem}{Theorems}
\crefname{lemma}{Lemma}{Lemmas}
\crefname{proposition}{Proposition}{Propositions}
\crefname{corollary}{Corollary}{Corollaries}
\crefname{conjecture}{Conjecture}{Conjectures}
\crefname{definition}{Definition}{Definitions}
\crefname{remark}{Remark}{Remarks}
\Crefname{theorem}{Theorem}{Theorems}
\Crefname{lemma}{Lemma}{Lemmas}
\Crefname{proposition}{Proposition}{Propositions}
\Crefname{corollary}{Corollary}{Corollaries}
\Crefname{conjecture}{Conjecture}{Conjectures}
\Crefname{definition}{Definition}{Definitions}
\Crefname{remark}{Remark}{Remarks}

\newcommand{\E}{\mathbb E}
\newcommand{\R}{\mathbb R}

\newcommand{\eps}{\varepsilon}
\newcommand{\dd}{\,\mathrm d}
\newcommand{\norm}[1]{\left\lVert #1\right\rVert}
\newcommand{\abs}[1]{\left\lvert #1\right\rvert}
\newcommand{\defeq}{\mathrel{:=}}

\title{\textbf{Fourth-Moment Geometry of Rademacher Sums}}

\author{Peigan Gao\thanks{The University of Hong Kong. \href{gaopg@mail.ustc.edu.cn}{\texttt{gaopg@hku.hk}}.}
\and
Jian Qian\thanks{
The University of Hong Kong. \href{mailto:jianqian@hku.hk}{\texttt{jianqian@hku.hk}}.
}
}

\date{August 18, 2026}

\begin{document}
\maketitle
\vspace{-1.5em}

\begin{abstract}
Let \(\varepsilon_1,\ldots,\varepsilon_n\) be independent Rademacher signs and let \(a=(a_1,\ldots,a_n)\in\R^n\) satisfy the normalization below.  For the normalized Rademacher sum,
we determine how its higher moments depend on the fourth-order mass.
Combining a sharp fixed-\(q\) moment envelope with a separate argument below
the convexity threshold gives the Gaussian stability inequality 
for the full range \(p\geq4\) of this linear-in-\(q\) bound.  The same fourth-order framework determines the
sharp finite dimensional \(L_p/L_4\) Khintchine constant for \(p\geq5\), with
the flat coefficient vector as the extremizer.  These results
settle the conjectures of Jakimiuk and of Bara\'nski, Murawski, Nayar, and
Oleszkiewicz stated below \cite{Jakimiuk, BMNO}.  We also prove Jakimiuk's conjectured quadratic
stability estimate at \(p=3\).  The resulting bounds retain information about
sparsity and effective dimension, with applications to Rademacher random
projections and randomly signed errors; those applications are not developed
further here.  Their Laplace-transform form also gives coefficient-sensitive
tail bounds. The proofs are discovered with substantial assistance from ChatGPT 5.6 Sol.
\end{abstract}

\section{Introduction}
\label{sec:introduction}

Let \(\varepsilon_1,\varepsilon_2,\ldots\) be independent Rademacher random
variables.  For a finitely supported vector \(a=(a_i)\in\R^n\), write
\[
S=\sum_{i=1}^n a_i\eps_i.
\]
We also use the unnormalized partial sums
\(S_m\defeq\sum_{i=1}^m\eps_i\).
For \(r>0\), write \(\|X\|_r\defeq(\E\abs X^r)^{1/r}\).  Rademacher sums are
fundamental models of randomly signed fluctuations.  They
appear throughout the probability and Banach space theory, as well as in the
Fourier analysis of Boolean functions and reduction in the dimension of the random sign 
\cite{ODonnell,Achlioptas}. Their moments are governed at first order by the
classical Khintchine inequality. Under normalization
\[
\sum_{i=1}^n a_i^2=1,
\]
its sharp upper bound is
\[
\mathbb{E}|S|^p\leq \mu_p,
\qquad
\mu_p:=\mathbb{E}|G|^p,
\qquad p>2,
\]
where \(G\) is standard Gaussian.  The determination of the
optimal constants goes back to Khintchine and culminated in Haagerup's theorem;
see \cite{Khintchine,Haagerup} and the later approach in
\cite{NayarOleszkiewicz}. By the classical central limit theorem, the Gaussian constant is approached by the flat sums
\[
\frac{\varepsilon_1+\cdots+\varepsilon_N}{\sqrt{N}}
\]
as \(N\to\infty\), but it is not attained by any nonzero finite Rademacher sum
when \(p>2\).

There is also a finite-dimensional endpoint.  The Schur comparison of Eaton
\cite{Eaton}, together with Komorowski's sharp finite-dimensional result
\cite{Komorowski}, shows that for fixed \(N\) and \(p\geq3\), the moment
\(\mathbb{E}|S_N|^p\) is maximized by the flat coefficient vector.  This does
not determine the finite-dimensional \(L_p/L_4\) ratio, because flattening
changes both moments simultaneously, but it identifies the second extremal
regime that will reappear below.

The classical Khintchine inequality, however, suppresses two pieces of information relevant here: the concentration behavior of \(S\) and the geometry of the optimizer in fixed dimension.
A natural parameter for both questions is
\[
q:=\sum_i a_i^4.
\]
Indeed,
\begin{equation}\label{eq:fourth-moment}
\mathbb{E}S^4=3-2\sum_i a_i^4=3-2q,
\qquad
\kappa_4(S)=-2\sum_i a_i^4=-2q,
\end{equation}
where \(\kappa_4(S)\) denotes the fourth cumulant.  Thus \(q\) directly
measures the fourth-moment defect from Gaussianity.  It also
records coefficient concentration: \(q=1\) for a coordinate vector, whereas
\(q=1/N\) for the flat vector in dimension \(N\).  Thus \(q^{-1}\) serves as an
effective number of active coordinates.  The central problem is to determine
the largest possible higher moment when this fourth-order mass is prescribed.

This problem belongs to the broader stability theory of sharp
Khintchine-type inequalities.  Earlier quantitative forms measured the distance
from the two-coordinate extremizers in the \(L_2/L_1\) inequality and led to
applications in Boolean Fourier analysis and convex geometry
\cite{DeDiakonikolasServedio,MelbourneRoberto,
EskenazisNayarTkoczResilience}. Distributional stability, in which the law of
the summands is perturbed, was developed in
\cite{EskenazisNayarTkoczDistributional}. Recent work also treats Gaussian and
coordinate stability for wider classes of symmetric laws at even exponents
\cite{ChavezSheng}.

\subsection{Main results}
The coefficient-stability question was introduced by Jakimiuk
\cite{Jakimiuk}.  For every \(p\geq3\), he proved that there is a constant
\(c_p>0\) such that
\[
\E\abs{\sum_{i=1}^n a_i\eps_i}^p
\le \mathbb{E}|G|^p-c_p\sum_{i=1}^n a_i^4
\]
and conjectured that the optimal coefficient is \(c_p=\mu_p-1\) for every
\(p\geq3\).  He also observed that the even-integer cases follow from moment
calculations.  The proposed range needs a correction at its lower endpoint.
For
\(S_2=(\eps_1+\eps_2)/\sqrt2\), one has \(q=1/2\), and the chord inequality fails for every
\(2<p<4\); see \Cref{prop:threshold}.  Ch\'avez and Sheng subsequently
obtained related stability estimates for even moments in a broader
distributional setting \cite{ChavezSheng}.  For Rademacher variables as a
special case, their Gaussian-side coefficient is \(2\mu_p/3\) for even integer
\(p\); this agrees with the sharp coefficient at \(p=4\) and is smaller for
\(p>4\).

We use the Gaussian-to-coordinate chord
\begin{equation}\label{eq:Lambda}
\Lambda_p(q)\defeq q+(1-q)\mu_p.
\end{equation}

The first main result is the following sharp stability bound.

\begin{theorem}[Sharp Gaussian stability]\label{thm:main-linear}
Let \(n\ge1\) and \(a\in\R^n\) satisfy \(\sum_{i=1}^n a_i^2=1\).  For every
real \(p\ge4\),
\begin{equation}\label{eq:main-linear}
\E\abs{\sum_{i=1}^n a_i\eps_i}^p
\le \mu_p-(\mu_p-1)\sum_{i=1}^n a_i^4
=\Lambda_p(q).
\end{equation}
The coefficient \(\mu_p-1\) is optimal.  At \(p=4\), equality holds for every coefficient
vector.  For \(p>4\), equality holds exactly for coordinate vectors, up to signs and
permutations.
\end{theorem}

The affine estimate is the strongest bound that is linear in \(q\), but it does
not identify the extremal law at an intermediate value of \(q\).  The finer
fixed-\(q\) problem reveals the structure behind the chord.  For
\(q\in(0,1]\), define
\[
\mathcal{A}(q)
:=
\left\{
a:\ a \text{ is finitely supported},\quad
\sum_i a_i^2=1,\quad
\sum_i a_i^4=q
\right\},
\]
and define
\[
Y_q:=q^{1/4}\varepsilon_0+\sqrt{1-\sqrt{q}}\,G,
\]
where \(\varepsilon_0\) and \(G\) are independent.  Thus \(Y_q\) has one
Bernoulli coordinate carrying all fourth-order mass and a Gaussian remainder
carrying the diffuse variance.

\begin{theorem}[Fixed-fourth-moment extremal principle]\label{thm:fixed-q-envelope}
Let \(\Phi\in C^4(\mathbb{R})\) be even and of polynomial growth, and suppose
that \(\Phi^{(4)}\) is convex. Then, for every \(q\in(0,1]\),
\begin{equation}\label{eq:fixed-q-envelope}
\sup_{a\in\mathcal{A}(q)}\E\Phi(S)=\E\Phi(Y_q).
\end{equation}  
In particular, for every \(p\geq5\),
\[
\sup_{a\in\mathcal{A}(q)}\mathbb{E}|S|^p
=
U_p(q)
:=
\mathbb{E}\left|
q^{1/4}\varepsilon_0+\sqrt{1-\sqrt{q}}\,G
\right|^p.
\]
Moreover, \(U_p\) is strictly convex on \((0,1)\).
\end{theorem}

The supremum in Theorem~\ref{thm:fixed-q-envelope} is generally attained only
in the closure of finite Rademacher sums.  One coefficient remains macroscopic,
while the remaining variance is split among increasingly many vanishing
coefficients.  Since
\[
U_p(0)=\mu_p,
\qquad
U_p(1)=1,
\]
strict convexity yields
\[
U_p(q)<(1-q)\mu_p+q,
\qquad 0<q<1.
\]
The theorem therefore explains the affine bound and shows that it is strict
throughout the interior.  The interval \(4<p<5\) is different: \(|x|^p\) does
not have a convex fourth derivative there, so the fixed-moment principle is
unavailable.  A separate argument completes
Theorem~\ref{thm:main-linear} in this interval.

The same fourth-order parameter organizes the finite-dimensional reverse
H\"older problem.  For \(N\ge1\) and \(p\ge4\), define the scale-invariant
constant
\[
 C_{p,4,N}\defeq
 \sup_{a\in\R^N\setminus\{0\}}
 \frac{\norm{\sum_{i=1}^N a_i\eps_i}_p}
      {\norm{\sum_{i=1}^N a_i\eps_i}_4},
\]
For the normalized sum
\(S=\sum_{i=1}^N a_i\eps_i\) with \(\sum_{i=1}^N a_i^2=1\),
\[
\|S\|_4^4=3-2q,
\]
so maximizing \(\|S\|_p/\|S\|_4\) couples the higher moment to precisely
the same fourth-order geometry. Bara\'nski, Murawski, Nayar, and Oleszkiewicz
proved that the dimension-free \(L_p/L_4\) Khintchine constant is Gaussian for
\(p\geq4\).  For \(p\geq5\), they reduced the finite-dimensional constant to
the one-parameter family
\[
C_{p,4,N+1}
=
\sup_{x\geq1}
\frac{\|x+\varepsilon_1+\cdots+\varepsilon_N\|_p}
     {\|x+\varepsilon_1+\cdots+\varepsilon_N\|_4}.
\]
They conjectured that the supremum is attained at \(x=1\)
\cite{BMNO}.  We prove the stronger monotonicity statement below, which
resolves that conjecture.

\begin{theorem}[Finite-dimensional \(L_p\)-\(L_4\) constant]\label{thm:finite-dimensional}
For every real \(p\ge5\) and every \(N\ge2\),
\begin{equation}\label{eq:finite-dimensional}
C_{p,4,N}=\frac{\norm{S_N}_p}{\norm{S_N}_4}.
\end{equation}
More strongly, for \(n=N-1\),
\begin{equation}\label{eq:shift-ratio-monotone}
x\longmapsto
\frac{\norm{x+S_n}_p}{\norm{x+S_n}_4}
\quad\text{is strictly decreasing on }[1,\infty).
\end{equation}
After normalization, the optimizer is unique up to coordinate signs and permutations; without normalization, also up to nonzero scalar multiplication.
\end{theorem}

This is the finite-dimensional counterpart of the Gaussian extremal regime.
As the number of coordinates grows, flat vectors converge to the Gaussian
endpoint \(q=0\).  In fixed dimension, the constraint
\[
q\geq \frac{1}{N}
\]
prevents that limit, and the least possible fourth-order mass is attained
exactly by the flat vector.  The fixed-\(q\) and finite-dimensional theorems
therefore resolve two versions of the same extremal problem: the former
identifies the extremal law after a Gaussian remainder is allowed, while the
latter identifies the extremal coefficient vector when only \(N\) coordinates
are available.

The preceding theorems do not cover the critical exponent \(p=3\).  For this
paragraph, let \(S=\sum_{i=1}^N a_i\eps_i\).  Jakimiuk
also asked whether diagonal stability could be strengthened to a
dimension-free quadratic deficit.  We prove that there is a universal
constant \(c>0\) such that
for every \(N\geq2\) and every normalized \(a\in\mathbb{R}^N\),
\[
\mathbb{E}|S|^3
\leq
\mathbb{E}\left|
\frac{\varepsilon_1+\cdots+\varepsilon_N}{\sqrt{N}}
\right|^3
-c\sum_{i=1}^N\left(a_i^2-\frac{1}{N}\right)^2.
\]
No linear-in-\(q\) chord statement is asserted here for the intermediate
range \(3<p<4\); the two-coordinate example above rules out the broader
range \(2<p<4\).
For the statement and proof below, write
\[
\overline S_n\defeq\frac{S_n}{\sqrt n}
=\frac1{\sqrt n}\sum_{i=1}^n\eps_i,
\qquad
\Delta_n(a)\defeq\sum_{i=1}^n\left(a_i^2-\frac1n\right)^2
=q-\frac1n.
\]

\begin{theorem}[Dimension-free stability at the third moment]\label{thm:third-moment}
For every \(n\ge1\) and every \(a\in\R^n\) with \(\sum_i a_i^2=1\),
\begin{equation}\label{eq:third-moment-main}
\E\abs{\sum_{i=1}^n a_i\eps_i}^3
\le \E\abs{\overline S_n}^3
-\frac{1}{100}\,\Delta_n(a).
\end{equation}
In particular, the optimal dimension-free constant
\begin{equation}\label{eq:C3-opt-definition}
C_3^{\mathrm{opt}}\defeq
\inf_{\substack{n\ge2,\;\sum_i a_i^2=1\\
\Delta_n(a)>0}}
\frac{\E\abs{\overline S_n}^3-
      \E\abs{\sum_i a_i\eps_i}^3}{\Delta_n(a)}
\end{equation}
satisfies
\begin{equation}\label{eq:C3-opt-bounds}
\frac{3(5\sqrt2-7)}{16}
\le C_3^{\mathrm{opt}}
\le 5\sqrt3-6\sqrt2.
\end{equation}
In addition, in dimension three, it satisfies
\[
\E\abs{\sum_{i=1}^3 a_i\eps_i}^3
\le \E\abs{\overline S_3}^3
-(5\sqrt3-6\sqrt2)\Delta_3(a).
\]
The upper bound is attained in dimension three by
\((a_1^2,a_2^2,a_3^2)=(1/2,1/2,0)\), up to signs and permutations.
\end{theorem}

This confirms the second conjecture in \cite{Jakimiuk} and describes the
endpoint at which the usual smooth-convexity argument degenerates.

The fixed-fourth-moment principle also yields coefficient-sensitive
concentration bounds.  Applying it to Taylor polynomials of \(\cosh\) and
passing to the limit gives, for every \(t\in\mathbb{R}\),
\[
\mathbb{E}e^{tS}
\leq
\exp\left(\frac{1-\sqrt{q}}{2}t^2\right)
\cosh\left(q^{1/4}t\right).
\]
Consequently, for \(u>0\),
\[
\mathbb{P}(|S|\geq u)
\leq
2\inf_{t>0}
\exp\left(
-tu+\frac{1-\sqrt{q}}{2}t^2
\right)
\cosh\left(q^{1/4}t\right).
\]

The right-hand side is the moment generating function of the
one-spike-plus-Gaussian extremizer \(Y_q\), and hence is optimal among
normalized Rademacher sums with prescribed fourth-order mass.  Chernoff's
method gives the corresponding refinement of the usual sub-Gaussian tail
bound.  The estimate distinguishes coefficient vectors with the same variance
but different effective support sizes, interpolating between the Gaussian
regime \(q\to0\) and a single random sign at \(q=1\).  In this sense the
extremal principle controls the full Laplace transform, not only individual
moments.

Theorems~\ref{thm:main-linear} and~\ref{thm:finite-dimensional} address
different extremal questions, and neither is a formal consequence of the
other.  Their proofs share a fourth-order input: for \(p\ge5\), extremizers
with fixed second and fourth moments have at most one exceptional coefficient.
Letting the number of small coefficients tend to infinity gives \(Y_q\) and
the Gaussian stability problem; keeping the dimension fixed gives the family
\((x,1,\ldots,1)\) and the finite-dimensional ratio problem.  The critical
third-moment theorem has a different structure.  Repeatedly averaging two
extreme squared coefficients produces a pointwise nonnegative cubic gain, and
a small-ball estimate makes that gain uniformly quadratic.
The new contributions are:
\begin{enumerate}
    \item the corrected real-\(p\) form of Jakimiuk's Conjecture 1 for \(p\ge4\), together with a counterexample for \(2<p<4\);
    \item a proof of the Bara\'nski--Murawski--Nayar--Oleszkiewicz flat-point conjecture; and
    \item the exact fixed-\(q\) moment and Laplace-transform envelopes;
    \item a proof of Jakimiuk's Conjecture 2 with an explicit universal constant.
\end{enumerate}

\subsection{Proof architecture}

\Cref{sec:direct} proves the stability theorem for \(4\le p\le6\) by a
coefficientwise argument related to the one in \cite{BMNO}; this route reaches
below the fourth-order convexity threshold.  For \(p\ge5\),
\Cref{sec:gaussian-shift} combines the fixed-moment reduction with a Riccati
comparison for noncentral Gaussian moments and completes the proof of
\Cref{thm:main-linear}.  \Cref{sec:fourth-order} records the exact fixed-\(q\)
envelope from the same reduction.  Next, \Cref{sec:finite-dim} uses the
finite-dimensional one-spike reduction of \cite{BMNO}; a determinant identity
and a cubic-quotient lemma show that every non-flat spike lowers the
\(L_p/L_4\) ratio.  Finally, \Cref{sec:third-moment} proves
\Cref{thm:third-moment} by repeatedly averaging the largest and smallest
squared coefficients.  The local gain comes from an explicit scalar smoothing
identity together with the Rademacher small-ball estimate of Dzindzalieta and
G\"otze.

\subsection{Statement of AI use}

Initial versions of the proofs of the sharp Gaussian stability inequality and
the finite-dimensional \(L_p\)-\(L_4\) constant theorem were developed with
assistance from ChatGPT 5.6 Sol.  The authors checked and revised the
arguments, take full responsibility for their mathematical content, and
independently verified the extensions presented here.

\section{A direct proof of sharp Gaussian stability for \texorpdfstring{$4\le p\le6$}{4 <= p <= 6}}\label{sec:direct}

\subsection{A coefficientwise leave-one-out estimate}

The proof of \Cref{thm:main-linear} uses two estimates with overlapping ranges.
When \(q\le1/2\), every coefficient is small enough for a uniform
coefficientwise bound.  When \(q\ge1/2\), interpolation between the exact
fourth moment and an upper bound for the sixth moment is stronger.  We use
\(q=1/2\) as a convenient dividing point.  Put
\[
x_i\defeq a_i^2,\qquad \sum_{i=1}^n x_i=1,\qquad
q=\sum_{i=1}^n x_i^2.
\]
The next identity isolates this fourth-power mass.  It is closely related to
the leave-one-out argument in \cite{BMNO}, but the coefficientwise form is
needed here.

\begin{lemma}\label{lem:leave-one-out}
Let \(p\ge4\).  Then
\begin{equation}\label{eq:leave-one-out}
\E\abs{S}^p
\le \mu_p\sum_{i=1}^n x_i\left(1-\frac{2x_i}{3}\right)^{(p-2)/2}.
\end{equation}
\end{lemma}

\begin{proof}
Fix \(i\), let \(U_i\) be uniform on \([-1,1]\) and independent of all signs,
and set
\[
S_i\defeq\sum_{j\ne i}a_j\eps_j+a_iU_i.
\]
With \(T_i=\sum_{j\ne i}a_j\eps_j\), the fundamental theorem of calculus gives,
for every continuously differentiable \(f\),
\[
\E_{\eps_i}\bigl[\eps_i f(T_i+a_i\eps_i)\bigr]
 =\frac{f(T_i+a_i)-f(T_i-a_i)}2
 =a_i\E_{U_i}f'(T_i+a_iU_i).
\]
After multiplication by \(a_i\), summation over \(i\), and expectation,
\[
\E[Sf(S)]=\sum_{i=1}^n a_i^2\E f'(S_i).
\]
Taking \(f(s)=\abs{s}^{p-2}s\) gives
\begin{equation}\label{eq:stein-identity}
\E\abs{S}^p=(p-1)\sum_{i=1}^n x_i\E\abs{S_i}^{p-2}.
\end{equation}
It remains to estimate each leave-one-out moment.  On an enlarged product
space, realize
\[
U_i=\sum_{k\ge1}2^{-k}\eps_{i,k}
\quad\text{almost surely},
\]
where the \(\eps_{i,k}\) are independent Rademacher variables, also independent of the original signs. Set
\[
S_i^{(m)}:=\sum_{j\neq i}a_j\eps_j+a_i\sum_{k=1}^{m}2^{-k}\eps_{i,k}.
\]
The sharp Khintchine inequality applied to this finite Rademacher sum gives
\[
\E|S_i^{(m)}|^{p-2}
\leq \mu_{p-2}\left(\sum_{j\neq i}a_j^2+a_i^2\sum_{k=1}^m4^{-k}\right)^{(p-2)/2}.
\]
As \(m\to\infty\), \(S_i^{(m)}\to S_i\) almost surely and
\(\sum_{k=1}^m4^{-k}\to1/3\).  Moreover,
\[
 |S_i^{(m)}|\le \sum_{j\ne i}|a_j|+|a_i|,
\]
uniformly in \(m\).  The right-hand side is a finite deterministic bound, so
its \((p-2)\)-power is integrable; dominated convergence therefore yields
\[
\E|S_i|^{p-2}
\leq \mu_{p-2}\left(1-\frac{2x_i}{3}\right)^{(p-2)/2}.
\]
Substituting this bound into \Cref{eq:stein-identity} and using
\((p-1)\mu_{p-2}=\mu_p\) proves \Cref{eq:leave-one-out}.
\end{proof}

To express the coefficientwise estimate only in terms of \(q\), we need a
uniform lower bound.  Since \((p-2)/2\ge1\), define
\begin{equation}\label{eq:phi-r}
\phi_p(x)\defeq
\frac{1-(1-2x/3)^{(p-2)/2}}{x}\quad(x>0),
\qquad \phi_p(0)\defeq\frac{p-2}{3}.
\end{equation}
The numerator is concave as a function of \(x\) and vanishes at the origin.
For a concave function \(f\) with \(f(0)=0\), the quotient \(f(x)/x\) is
nonincreasing; hence \(\phi_p\) is nonincreasing on \([0,1]\).

\begin{lemma}\label{lem:small-q-scalar}
For \(p\ge4\) and \(x_0=2^{-1/2}\),
\begin{equation}\label{eq:small-q-scalar}
\phi_p(x_0)\ge1-\frac1{\mu_p}.
\end{equation}
For \(p>4\), the inequality is strict.
\end{lemma}

A self-contained proof is given in \Cref{app:scalar-small}.

\begin{proposition}\label{prop:small-q}
If \(p\ge4\) and \(q\le1/2\), then \Cref{eq:main-linear} holds.
\end{proposition}

\begin{proof}
Let \(m_*=\max_i x_i\).  The assumption \(q\le1/2\) and the inequality
\(m_*^2\le\sum_i x_i^2=q\) give \(m_*\le2^{-1/2}\).  Monotonicity of
\(\phi_p\) and \Cref{lem:small-q-scalar} now give
\begin{align*}
1-\sum_i x_i\left(1-\frac{2x_i}{3}\right)^{(p-2)/2}
 &=\sum_i x_i^2\phi_p(x_i)\geq \sum_i x_i^2 \phi_p(m_*)\\
 &\ge q\phi_p(m_*)
 \ge q\phi_p(2^{-1/2})
 \ge q\left(1-\frac1{\mu_p}\right).
\end{align*}
Multiplication by \(\mu_p\), followed by \Cref{lem:leave-one-out}, gives
\(\E\abs S^p\le\Lambda_p(q)\).
\end{proof}

\subsection{Interpolation between the fourth and sixth moments}

Set
\begin{equation}\label{eq:A-B}
A(q)\defeq3-2q,
\qquad
B(q)\defeq15-30q+16q^{3/2}.
\end{equation}
The exact sixth-moment formula and the inequality
\[
\sum_i x_i^3\le\left(\sum_i x_i^2\right)^{3/2}=q^{3/2},
\]
which is the monotonicity of finite-dimensional \(\ell_r\)-norms, give
\begin{equation}\label{eq:sixth-upper}
\E S^6
 =15-30q+16\sum_i x_i^3
 \le15-30q+16q^{3/2}=B(q).
\end{equation}

\begin{lemma}\label{lem:large-q-scalar}
For \(1/2\le q\le1\) and \(4 \le p \le 6\)

\begin{equation}\label{eq:interpolant-chord}
(3-2q)^{(6-p)/2}(15-30q+16q^{3/2})^{(p-4)/2}\le q+(1-q)\mu_p =\Lambda_p(q).
\end{equation}

\end{lemma}

The elementary calculus proof appears in \Cref{app:scalar-large}.

\begin{proposition}\label{prop:large-q}
If \(4\le p\le6\) and \(q\ge1/2\), then \Cref{eq:main-linear} holds.
\end{proposition}

\begin{proof}
Fix \(q\in[1/2,1]\).  For \(4\le p\le6\), log-convexity between the fourth
and sixth moments, together with \(\E S^4=3-2q\), gives
\[
\E\abs{S}^p
\le(\E S^4)^{(6-p)/2}(\E S^6)^{(p-4)/2}
\le A(q)^{(6-p)/2}B(q)^{(p-4)/2}.
\]
The scalar bound \Cref{eq:interpolant-chord} then yields
\(\E\abs S^p\le\Lambda_p(q)\).
\end{proof}

\begin{corollary}\label{cor:direct-range}
The inequality \Cref{eq:main-linear} holds for every \(4\le p\le6\).
\end{corollary}

\begin{proof}
Combine \Cref{prop:small-q} when \(q\le1/2\) with \Cref{prop:large-q} when \(q\ge1/2\).
\end{proof}

\section{The fixed-moment reduction for \texorpdfstring{$p\geq5$}{p >= 5}}\label{sec:gaussian-shift}

For \(p\ge5\), the extremal principle in \Cref{lem:BMNO} reduces an arbitrary
Rademacher sum to one distinguished coefficient together with equal smaller
coefficients.  Passing to infinitely many small coefficients leaves one
Gaussian shift, so the remaining comparison is one-dimensional.  This gives a
second proof of \Cref{thm:main-linear} in that range.

\subsection{Reduction to a Gaussian shift}

\begin{lemma}[\cite{BMNO}]\label{lem:BMNO}
Let \(N\ge2\), let \(\Phi\in C^4(\R)\) be even with convex fourth derivative, and fix feasible
numbers \(\sigma_2,\sigma_4\).  On
\[
\left\{b\in\R^N:\sum_{i=1}^N b_i^2=\sigma_2,
                         \ \sum_{i=1}^N b_i^4=\sigma_4\right\},
\]
the functional
\[
b\longmapsto\E\Phi\left(\sum_{i=1}^N b_i\eps_i\right)
\]
has a maximizer of the form \((c,d,\ldots,d)\), where \(c\ge d\ge0\). 
\end{lemma}

For \(\Phi(x)=\abs{x}^p\), the hypothesis holds whenever \(p\ge5\), because
\[
\Phi^{(4)}(x)=p(p-1)(p-2)(p-3)\abs{x}^{p-4}
\]
is convex.

\begin{proposition}[Gaussian reduction]\label{prop:gaussian-reduction}
Let \(p\ge5\).  Then
\begin{equation}\label{eq:gaussian-reduction}
\E\abs{S}^p
\le\E\abs{q^{1/4}+\sqrt{1-\sqrt q}\,G}^p
=U_p(q).
\end{equation}
\end{proposition}

\begin{proof}
Append zeros and regard the coefficient vector as an element of \(\R^N\), where
\(N\ge n\).  For each such \(N\), \Cref{lem:BMNO} gives
\[
\E\abs{S}^p
\le\E\abs{c_N\eps_1+d_N\sum_{i=2}^N\eps_i}^p,
\]
where \(c_N\ge d_N\ge0\) satisfy
\[
c_N^2+(N-1)d_N^2=1,
\qquad
c_N^4+(N-1)d_N^4=q.
\]
The two constraints can be solved explicitly.  Set
\begin{equation}\label{eq:delta-cd}
\delta_{N,q}\defeq\sqrt{\frac{Nq-1}{N-1}}.
\end{equation}
Then
\begin{equation}\label{eq:cd-formulas}
c_N^2=\frac{1+(N-1)\delta_{N,q}}{N},
\qquad
d_N^2=\frac{1-\delta_{N,q}}{N}.
\end{equation}
As \(N\to\infty\),
\[
c_N\longrightarrow q^{1/4},
\qquad
(N-1)d_N^2\longrightarrow1-\sqrt q,
\qquad
d_N\longrightarrow0.
\]
The Lindeberg--Feller central limit theorem therefore gives
\[
d_N\sum_{i=2}^N\eps_i\xrightarrow{D}\sqrt{1-\sqrt q}\,G.
\]
The leading sign is independent of the remaining sum, and
\(c_N\eps_1\to q^{1/4}\eps_1\) almost surely.  Thus, with an independent copy
\(\eps_0\),
\[
X_N\defeq c_N\eps_1+d_N\sum_{i=2}^N\eps_i
\xrightarrow{D} q^{1/4}\eps_0+\sqrt{1-\sqrt q}\,G.
\]
To pass from convergence in distribution to convergence of the \(p\)-th
moments, apply the sharp Khintchine inequality at exponent \(p+1\):
\[
  \sup_{N\ge n}\E|X_N|^{p+1}\le\mu_{p+1}.
  \]
Thus \(\{|X_N|^p\}_{N\ge n}\) is uniformly integrable, and convergence in
distribution implies
\[
\lim_{N\to\infty}\E\abs{X_N}^p
=\E\abs{q^{1/4}\eps_0+\sqrt{1-\sqrt q}\,G}^p.
\]
Since \(\E\abs{S}^p\le\E\abs{X_N}^p\) for every \(N\ge n\), the limit above,
together with the symmetry of \(G\), proves \Cref{eq:gaussian-reduction}.
\end{proof}

\subsection{The Gaussian-shift chord}

\begin{lemma}[Gaussian-shift chord]\label{lem:gaussian-chord}
Let \(p>4\), \(0\le u\le1\), and \(G\sim N(0,1)\).  Then
\begin{equation}\label{eq:gaussian-chord}
\E\abs{u+\sqrt{1-u^2}\,G}^p
\le\mu_p-(\mu_p-1)u^4.
\end{equation}
For \(0<u<1\), the inequality is strict.
\end{lemma}

\begin{proof}
Consider the fixed-moment profile
\[
g(s)\defeq
\E\abs{s^{1/4}+\sqrt{1-\sqrt s}\,G}^p,
\qquad 0\le s\le1.
\]
Its endpoint values are \(g(0)=\mu_p\) and \(g(1)=1\).  Strict convexity on
\((0,1)\) would therefore place \(g\) below the chord joining these endpoints,
which is precisely \Cref{eq:gaussian-chord}.

To verify the required convexity, write \(u=s^{1/4}\) and
\[
h(u)=\E\abs{u+\sqrt{1-u^2}\,G}^p.
\]
Differentiating gives
\begin{equation}\label{eq:g-second}
g''(s)=\frac{uh''(u)-3h'(u)}{16u^7}.
\end{equation}
Hence the sign of \(g''\) is determined by \(uh''(u)-3h'(u)\), and it remains
to prove
\(uh''(u)-3h'(u)>0\) for \(0<u<1\).

For the curvature calculation, put
\[
v=\sqrt{1-u^2},\qquad t=\frac{u}{v}.
\]
Since \(p>4\), set \(\alpha=p-4\) and
\[
m(t)=\E\abs{G+t}^{\alpha}.
\]
Moreover, the Gaussian convolution is smooth in t. Gaussian integration by parts gives
\begin{equation}\label{eq:gaussian-ode}
m''(t)+tm'(t)-\alpha m(t)=0.
\end{equation}
Using this identity, the curvature reduces to
\begin{align}
uh''(u)-3h'(u)
={}&p(\alpha+2)v^{\alpha+3}
\Bigl\{
\alpha t\bigl(3+(\alpha+2)t^2\bigr)m(t)\notag\\
&\hspace{29mm}
-\bigl(3+(2\alpha+3)t^2\bigr)m'(t)
\Bigr\}.
\label{eq:curvature-before-Q}
\end{align}
The calculation is given in \Cref{app:gaussian-chord-curvature}.  The bracket
is positive exactly when
\begin{equation}\label{eq:gaussian-log-derivative-comparison}
\frac{m'(t)}{t m(t)}
<
\alpha\frac{3+(\alpha+2)t^2}{3+(2\alpha+3)t^2},
\qquad t>0.
\end{equation}

Write \(R(t)\) for the logarithmic derivative on the left of
\Cref{eq:gaussian-log-derivative-comparison}, and \(Q(t)\) for the explicit
function on the right.  The function \(Q\) is the threshold at which the
bracket in \Cref{eq:curvature-before-Q} changes sign.  The differential equation
\Cref{eq:gaussian-ode} gives
\begin{equation}\label{eq:R-riccati}
tR'(t)=\alpha-(1+t^2)R(t)-t^2R(t)^2.
\end{equation}
Direct substitution shows that \(Q\) is a strict supersolution of the same
Riccati equation:
\begin{equation}\label{eq:Q-supersolution}
tQ'(t)>
\alpha-(1+t^2)Q(t)-t^2Q(t)^2,
\qquad t>0.
\end{equation}
The calculation is recorded in \Cref{app:gaussian-chord-comparison}.  Since
\(m\) is even and \Cref{eq:gaussian-ode} gives
\(m''(0)=\alpha m(0)\), both \(R(t)\) and \(Q(t)\) tend to \(\alpha\) as
\(t\downarrow0\).  The positive remainder in the supersolution inequality
shows that \(Q-R\) satisfies a first-order linear equation with positive
forcing.  Because \(Q(t)-R(t)\to0\) at the origin, variation of constants gives
\[
R(t)<Q(t),\qquad t>0.
\]
The endpoint comparison and variation-of-constants formula are written out in
\Cref{app:gaussian-chord-comparison}.

Hence the right-hand side of \Cref{eq:curvature-before-Q} is positive.  By
\Cref{eq:g-second}, \(g''(s)>0\) on \((0,1)\).  The endpoint values of \(g\)
then give \Cref{eq:gaussian-chord}, with strict inequality for \(0<s<1\).
\end{proof}

\begin{proposition}\label{prop:fixed-range}
The inequality \Cref{eq:main-linear} holds for every \(p\ge5\).
\end{proposition}

\begin{proof}
Apply \Cref{prop:gaussian-reduction}, followed by \Cref{lem:gaussian-chord} with
\(u=q^{1/4}\).
\end{proof}

\begin{proof}[Proof of \Cref{thm:main-linear}]
\Cref{cor:direct-range} covers \(4\le p\le6\), while \Cref{prop:fixed-range} covers
\(p\ge5\).  Their union covers the full range \(p\ge4\).  At \(p=4\), the theorem is exactly
\Cref{eq:fourth-moment}.  Coordinate vectors give equality for every \(p>4\),
so the coefficient \(\mu_p-1\) cannot be increased.  The equality cases are
proved in \Cref{cor:equality}.
\end{proof}

\section{The exact fixed-\texorpdfstring{$q$}{q} upper envelope}
\label{sec:fourth-order}

This section records the full upper envelope supplied by the fixed-moment
principle, rather than only its affine consequence.

\subsection{Finite-dimensional extremizers}

Fix \(N\ge2\) and a feasible \(q\in[1/N,1]\).  Let \(c_{N,q},d_{N,q}\) be given by
\Cref{eq:delta-cd,eq:cd-formulas}, and define
\begin{equation}\label{eq:finite-upper-law}
Y_{N,q}^+\defeq c_{N,q}\eps_1+d_{N,q}\sum_{j=2}^N\eps_j.
\end{equation}
Then
\[
\sum_{i=1}^N a_i^2=1,
\qquad
\sum_{i=1}^N a_i^4=q
\]
for the coefficients of \(Y_{N,q}^+\).

\begin{proposition}[Exact finite-dimensional optimization]\label{prop:finite-fixed-q}
Let \(\Phi\in C^4(\R)\) be even with convex fourth derivative, let \(N\ge2\), and let \(q\in[1/N,1]\).  Then
\begin{equation}
\max_{\substack{a\in\R^N\\\sum_i a_i^2=1,\ \sum_i a_i^4=q}}
\E\Phi\left(\sum_i a_i\eps_i\right)
 =\E\Phi(Y_{N,q}^+).
\label{eq:finite-max}
\end{equation}
\end{proposition}

\begin{proof}
\Cref{lem:BMNO} reduces the maximizer to coefficients
\((c,d,\ldots,d)\).  Solving
\[
c^2+(N-1)d^2=1,
\qquad c^4+(N-1)d^4=q,
\qquad c\ge d\ge0,
\]
gives \Cref{eq:cd-formulas} and hence the stated maximum.
\end{proof}

\subsection{Passage to one spike and a Gaussian sea}

\begin{proof}[Proof of \Cref{thm:fixed-q-envelope}]
Fix \(N\ge n\), append zeros to the coefficient vector of \(S\), and apply
\Cref{prop:finite-fixed-q}:
\[
\E\Phi(S)\le\E\Phi(Y_{N,q}^+).
\]
By \Cref{eq:cd-formulas},
\[
c_{N,q}\to q^{1/4},
\qquad
(N-1)d_{N,q}^2\to1-\sqrt q,
\qquad
d_{N,q}\to0.
\]
The Lindeberg--Feller and independence argument from
\Cref{prop:gaussian-reduction} gives \(Y_{N,q}^+\xrightarrow{D}Y_q\).  If
\(\Phi\) has polynomial growth of degree at most \(m\), the sharp Khintchine
inequality uniformly bounds a moment of order strictly larger than \(m\).
Thus \(\{\Phi(Y_{N,q}^+)\}\) is uniformly integrable, and
\[
\E\Phi(S)\le\lim_{N\to\infty}\E\Phi(Y_{N,q}^+)=\E\Phi(Y_q).
\]
Conversely, \(Y_{N,q}^+\) is feasible whenever \(Nq\ge1\).  Its expectations
converge to \(\E\Phi(Y_q)\), so the upper bound is approached by finite
Rademacher sums and is therefore the exact supremum.
\end{proof}

\subsection{Exact moment profiles and strict improvements}

\begin{proof}[Proof of the moment-profile and convexity assertions in
\Cref{thm:fixed-q-envelope}]
For \(p\ge5\), \(\Phi(x)=\abs{x}^p\) satisfies the hypotheses of
\Cref{thm:fixed-q-envelope}, giving the fixed-\(q\) supremum formula.  The
Riccati argument in \Cref{lem:gaussian-chord} proves that
the profile \(g(s)\) in \Cref{lem:gaussian-chord} is strictly convex for
\(p>4\).  Since \(s\) is the same fourth-moment parameter \(q\), this is
strict convexity of \(q\mapsto U_p(q)\).  Together with
\(U_p(0)=\mu_p\) and \(U_p(1)=1\), this gives
\(U_p(q)<\Lambda_p(q)\) for \(0<q<1\).
\end{proof}

\begin{corollary}[Equality cases]\label{cor:equality}
At \(p=4\), equality in \Cref{eq:main-linear} holds for every coefficient vector.  For every
\(p>4\), equality holds if and only if \(q=1\), equivalently, exactly one coefficient is nonzero
and has magnitude \(1\).
\end{corollary}

\begin{proof}
The case \(p=4\) is \Cref{eq:fourth-moment}.  For \(p\ge5\), strict
convexity in \Cref{thm:fixed-q-envelope} gives strictness whenever
\(0<q<1\).  For \(4<p<5\), \Cref{lem:small-q-scalar} makes the proof of
\Cref{prop:small-q} strict when \(q\le1/2\).  When \(1/2\le q<1\), the
supporting-tangent comparison that leads from \Cref{eq:large-q-log} to
\Cref{eq:interpolant-chord} is strict.  Finally, under
\(\sum_i a_i^2=1\), the condition \(q=1\) is equivalent to a coordinate
vector.
\end{proof}

The following two-coordinate example shows why the original conjectured range
cannot hold.
\begin{proposition}\label{prop:threshold}
For every \(2<p<4\), the inequality \Cref{eq:main-linear} fails for
\(S_2=(\eps_1+\eps_2)/\sqrt2\).
\end{proposition}

\begin{proof}
Since \(G^2\) is nonconstant, strict log-convexity of its moments gives
\[
\mu_p=\E(G^2)^{p/2}
<(\E G^2)^{2-p/2}(\E G^4)^{p/2-1}=3^{p/2-1}.
\]
Since \(0<p/2-1 <1\), the concavity of \(x\mapsto x^{p/2-1}\) yields
\[
\frac{1+3^{p/2-1}}{2}<\left(\frac{1+3}{2}\right)^{p/2-1}=2^{p/2-1}.
\]
Combining these two strict inequalities gives \(1+\mu_p<2^{p/2}\).  For
\(S_2\), one has \(q=1/2\) and \(\E\abs{S_2}^p=2^{p/2-1}\), hence
\[
\E\abs{S_2}^p>\frac{\mu_p+1}{2}.
\]
\end{proof}

\subsection{Laplace transform}
The fixed-\(q\) principle also determines the sharp moment generating function
of \(S\).
\begin{theorem}[Sharp fixed-\(q\) Laplace transform]\label{thm:mgf}
Let \(q\in(0,1]\), let \(a\in\mathcal A(q)\), and write
\(S=\sum_i a_i\eps_i\).  Then, for every \(t\in\R\),
\begin{equation}\label{eq:mgf}
\E e^{tS}=\prod_i\cosh(ta_i)
\le \exp\left(\frac{1-\sqrt q}{2}t^2\right)
\cosh(q^{1/4}t).
\end{equation}
More precisely, the right-hand side equals
\(\sup_{a\in\mathcal A(q)}\E\exp(t\sum_i a_i\eps_i)\); it is the exact
supremum at fixed \(q\), approached in the finite-sum closure.
\end{theorem}

\begin{proof}
Since \(S\) is symmetric, \(\E e^{tS}=\E\cosh(tS)\).  Apply
\Cref{lem:BMNO} to \(\Phi(x)=\cosh(tx)\), whose fourth derivative
\(t^4\cosh(tx)\) is convex.  After appending zeros, for every \(N\ge n\),
\[
\E e^{tS}
\le \cosh(tc_{N,q})\cosh(td_{N,q})^{N-1}.
\]
By \Cref{eq:cd-formulas}, \((N-1)d_{N,q}^4\to0\).  Hence the expansion
\(\log\cosh z=z^2/2+O(z^4)\) at the origin gives
\[
\cosh(tc_{N,q})\cosh(td_{N,q})^{N-1}
\longrightarrow
\cosh(q^{1/4}t)\exp\left(\frac{1-\sqrt q}{2}t^2\right).
\]
Because the finite upper vectors are feasible and converge to the limiting
law, the right-hand side is the exact supremum.
\end{proof}

\section{The finite-dimensional \texorpdfstring{$L_p$--$L_4$}{Lp--L4} conjecture}
\label{sec:finite-dim}

Bara\'nski, Murawski, Nayar, and Oleszkiewicz proved, for \(p\ge5\), that
\begin{equation}\label{eq:BMNO-reduction}
C_{p,4,n+1}
=\sup_{x\ge1}
\frac{\norm{x+S_n}_p}{\norm{x+S_n}_4},
\end{equation}
and conjectured that the supremum is attained at \(x=1\) \cite{BMNO}. We prove the stronger
monotonicity statement \Cref{eq:shift-ratio-monotone}, which resolves the conjecture.

\subsection{A determinant identity}

For an integer \(r\ge0\) and \(c\in\R\), set
\begin{equation}\label{eq:J-H}
J_{p,r}(c)\defeq
\E\bigl[(c+S_r)\abs{c+S_r}^{p-2}\bigr],
\qquad
H_r(c)\defeq\E(c+S_r)^3=c(c^2+3r).
\end{equation}
For fixed \(n\ge1\), define
\[
A_p(x)\defeq\E\abs{x+S_n}^p,
\qquad
B(x)\defeq\E\abs{x+S_n}^4.
\]
For use below, set \(S_0=0\) and put \(r=n-1\).

\begin{lemma}[Determinant identity]\label{lem:determinant}
For \(x>1\),
\begin{equation}\label{eq:determinant}
4B(x)\frac{\partial A_p}{\partial(x^2)}(x)
-pA_p(x)\frac{\partial B}{\partial(x^2)}(x)
=\frac{pn}{x}
\bigl[H_r(x+1)J_{p,r}(x-1)-H_r(x-1)J_{p,r}(x+1)\bigr].
\end{equation}
\end{lemma}

\begin{proof}
The identity follows by expressing both moments through the two neighboring
shifts \(x-1\) and \(x+1\).  Condition on the last sign in \(S_n\).
Exchangeability gives
\begin{align}
A_p(x)
&=\frac12\bigl[(x+n)J_{p,r}(x+1)+(x-n)J_{p,r}(x-1)\bigr],
\label{eq:A-decomposition}\\
B(x)
&=\frac12\bigl[(x+n)H_r(x+1)+(x-n)H_r(x-1)\bigr].
\label{eq:B-decomposition}
\end{align}
To see the first identity, write
\(\abs{X}^p=X\cdot X\abs{X}^{p-2}\), where \(X=x+S_n\), and separate the
deterministic part \(x\) from the sum of signs; exchangeability handles the
sum.  The second identity follows in the same way from \(X^4=X\cdot X^3\).

Differentiating the two shift representations, and using
\(\partial_{x^2}=(2x)^{-1}\partial_x\), yields
\[
\frac{\partial A_p}{\partial(x^2)}
=\frac{p}{4x}\bigl(J_{p,r}(x-1)+J_{p,r}(x+1)\bigr),
\qquad
\frac{\partial B}{\partial(x^2)}
=\frac1x\bigl(H_r(x-1)+H_r(x+1)\bigr).
\]
Substitution of \Cref{eq:A-decomposition,eq:B-decomposition} cancels the
diagonal products and leaves \Cref{eq:determinant}.
\end{proof}

Since
\[
\frac{\dd}{\dd(x^2)}
\log\frac{A_p(x)^{1/p}}{B(x)^{1/4}}
=\frac{4B\,\partial_{x^2}A_p-pA_p\,\partial_{x^2}B}{4pA_pB},
\]
the desired decrease of the norm ratio will follow once \(J_{p,r}/H_r\) is
shown to be increasing.

\subsection{The cubic quotient lemma}

\begin{lemma}[Cubic quotient]\label{lem:cubic-quotient}
Let \(g\in C^3(\R)\) be odd, and suppose that \(h=g'''\) is even and convex.  Let \(a\ge0\)
and assume
\begin{equation}\label{eq:cubic-initial}
a h(0)\ge6g'(0).
\end{equation}
Then
\begin{equation}\label{eq:cubic-quotient}
c\longmapsto\frac{g(c)}{c(c^2+a)}
\end{equation}
is nondecreasing on \((0,\infty)\).  It is strictly increasing if
\(a h(0)>6g'(0)\).
\end{lemma}

\begin{proof}
The proof turns the derivative of the quotient into two successive
monotonicity statements.  Its numerator is
\[
N(c)=c(c^2+a)g'(c)-(3c^2+a)g(c).
\]
To control \(N\), set
\[
M(c)=(c^2+a)g''(c)-6g(c).
\]
Differentiation gives the useful relation
\begin{equation}\label{eq:N-prime}
N'(c)=cM(c).
\end{equation}
Since \(g\) is odd, \(M(0)=0\), and
\[
M'(0)=a h(0)-6g'(0)\ge0.
\]
Because the convex function \(h\) is locally absolutely continuous, the
following computation holds for almost every \(c>0\):
\begin{align}
M''(c)
&=(c^2+a)h'(c)
+4\left(ch(c)-\int_0^c h(s)\dd s\right)
\ge0.
\label{eq:M-second}
\end{align}
Here evenness and convexity make \(h\) nondecreasing on \([0,\infty)\), which
also makes the last integrand nonnegative.  Hence \(M''\ge0\) almost
everywhere.  Since \(M'(0)\ge0\), the function \(M'\) is nonnegative and
\(M(c)\ge0\) for \(c\ge0\).  Equation~\ref{eq:N-prime}, together with
\(N(0)=0\), now gives \(N(c)\ge0\).  If the initial inequality is strict, then
\(M'(0)>0\), and the same chain gives \(N(c)>0\) for every \(c>0\).
\end{proof}

The initial condition in \Cref{lem:cubic-quotient} will follow from the next
real-moment recurrence.  Its proof uses the same conditioning idea as
\Cref{lem:leave-one-out}.

\begin{lemma}[Rademacher moment recurrence]\label{lem:moment-recurrence}
For every integer \(r\ge1\) and every real \(s\ge3\),
\begin{equation}\label{eq:moment-recurrence}
\E\abs{S_r}^s
\le(s-1)r\,\E\abs{S_r}^{s-2}.
\end{equation}
\end{lemma}

\begin{proof}
Let \(F_s(y)=y\abs{y}^{s-2}\).  Expanding \(S_r\) coefficient by coefficient,
\[
\E\abs{S_r}^s
=\E[S_rF_s(S_r)]
=\sum_{i=1}^r\E[\eps_iF_s(S_r)].
\]
Conditioning on \(T_i=S_r-\eps_i\), the fundamental theorem of calculus gives
\[
\E_{\eps_i}[\eps_iF_s(T_i+\eps_i)]
=\frac{F_s(T_i+1)-F_s(T_i-1)}2
=\frac{s-1}{2}\int_{T_i-1}^{T_i+1}\abs{z}^{s-2}\dd z.
\]
The function \(z\mapsto\abs{z}^{s-2}\) is convex.  The upper
Hermite--Hadamard inequality therefore gives
\[
\frac12\int_{T_i-1}^{T_i+1}\abs{z}^{s-2}\dd z
\le\frac{\abs{T_i-1}^{s-2}+\abs{T_i+1}^{s-2}}2.
\]
After expectation, the right-hand side equals
\(\E\abs{S_r}^{s-2}\) by symmetry of the conditioned sign.  Summing over
\(i=1,\ldots,r\) proves \Cref{eq:moment-recurrence}.
\end{proof}

\subsection{Proof of the finite-dimensional theorem}

\begin{proof}[Proof of \Cref{thm:finite-dimensional}]
By \Cref{eq:BMNO-reduction}, the theorem reduces to
\Cref{eq:shift-ratio-monotone}.  First suppose \(n\ge2\), so \(r=n-1\ge1\).
We apply \Cref{lem:cubic-quotient} to
\[
g(c)=J_{p,r}(c),
\qquad a=3r.
\]
Indeed, \(g\) is the convolution of the odd function
\(f_p(x)=x\abs{x}^{p-2}\), and
\[
f_p'''(x)=(p-1)(p-2)(p-3)\abs{x}^{p-4}.
\]
with the symmetric law of \(S_r\).  For \(p\ge5\), the displayed third
derivative is even and convex; convolution preserves both properties, so
\(g'''\) is even and convex.

Let \(s=p-2\ge3\).  Differentiation under the finite expectation gives
\[
g'(0)=(p-1)\E\abs{S_r}^s,
\qquad
g'''(0)=(p-1)s(s-1)\E\abs{S_r}^{s-2}.
\]
The moment recurrence now verifies the strict initial condition:
\begin{align*}
3r g'''(0)-6g'(0)
&\ge3r(p-1)(s-1)(s-2)\E\abs{S_r}^{s-2}>0.
\end{align*}
The cubic-quotient lemma therefore shows that
\begin{equation}\label{eq:JH-increasing}
c\longmapsto\frac{J_{p,r}(c)}{H_r(c)}
=\frac{J_{p,r}(c)}{c(c^2+3r)}
\end{equation}
is strictly increasing on \((0,\infty)\).

For \(x>1\), one has \(0<x-1<x+1\), and hence
\[
\frac{J_{p,r}(x-1)}{H_r(x-1)}
<\frac{J_{p,r}(x+1)}{H_r(x+1)},
\]
After cross multiplication by the positive values of \(H_r\), this is exactly
the negativity of the bracket in \Cref{eq:determinant}.  Therefore,
\[
\frac{\dd}{\dd(x^2)}
\log\frac{\norm{x+S_n}_p}{\norm{x+S_n}_4}<0,
\qquad x>1.
\]
The ratio is strictly decreasing on \((1,\infty)\).

When \(n=1\), use the same determinant identity with
\(J_{p,0}(c)=c^{p-1}\) and \(H_0(c)=c^3\); their quotient is \(c^{p-4}\), which is strictly
increasing for \(p>4\).  The ratio is therefore strictly decreasing on
\((1,\infty)\) in this case as well.

It remains to include \(x=1\).  For \(s\in\{4,p\}\), the reverse triangle
inequality gives
\[
\left|
\norm{x+S_n}_s-\norm{y+S_n}_s
\right|
\leq |x-y|.
\]
Hence \(x\mapsto\norm{x+S_n}_s\) is continuous.  Since
\(\norm{1+S_n}_4>0\), their ratio
\[
x\longmapsto
\frac{\norm{x+S_n}_p}{\norm{x+S_n}_4}
\]
is continuous at \(x=1\).  For any \(x>1\), choose \(y\in(1,x)\).  Strict
decrease gives \(f(x)<f(y)\), where \(f\) denotes the ratio, and continuity
gives \(f(y)\to f(1)\) as \(y\downarrow1\).  Hence
\[
\frac{\norm{x+S_n}_p}{\norm{x+S_n}_4}
<
\frac{\norm{1+S_n}_p}{\norm{1+S_n}_4},
\qquad x>1.
\]
The reduction \Cref{eq:BMNO-reduction} now gives
\[
C_{p,4,N}
=\frac{\norm{1+S_{N-1}}_p}{\norm{1+S_{N-1}}_4}
=\frac{\norm{S_N}_p}{\norm{S_N}_4},
\]
where the last equality follows by restoring the independent leading sign.

For uniqueness, normalize an optimizer by \(\sum_i a_i^2=1\) and put
\(q=\sum_i a_i^4\).  At this fixed \(q\), \Cref{lem:BMNO} bounds its ratio by
that of \(Y_{N,q}^+\).  If \(q>1/N\), the upper vector has coefficient ratio
\(c_{N,q}/d_{N,q}>1\).  At \(q=1\), \(d_{N,q}=0\) and the upper vector is
already a coordinate vector; otherwise the ratio is finite and strictly
larger than one. Its \(L_p/L_4\) ratio equals \(1\). On the
other hand, since \(N\geq 2\), the random variable
\(\lvert S_N\rvert\) is nonconstant. Hence, for \(p>4\), strict
monotonicity of \(L_r\)-norms gives
$
\frac{\|S_N\|_p}{\|S_N\|_4}>1.$
Thus the coordinate vector is also strictly suboptimal. Consequently,
every optimizer must satisfy \(q=1/N\). Equality in Cauchy--Schwarz
forces \(\abs{a_1}=\cdots=\abs{a_N}\), proving uniqueness up to signs and
permutations.
\end{proof}

\begin{corollary}[Explicit fixed-dimensional constant]\label{cor:explicit-finite}
For \(p\ge5\) and \(N\ge2\),
\begin{equation}\label{eq:explicit-finite}
C_{p,4,N}
=\frac{
\left(2^{-N}\sum_{k=0}^N\binom Nk\abs{N-2k}^p\right)^{1/p}}
{(3N^2-2N)^{1/4}}.
\end{equation}
After normalization, the flat coefficient vector is the unique optimizer up to coordinate signs
and permutations; without normalization, its nonzero scalar multiples are also optimizers.
\end{corollary}

\begin{proof}
The formula follows from the binomial law of \(S_N\) and
\(\E S_N^4=3N^2-2N\).  Uniqueness was proved at the end of the proof of
\Cref{thm:finite-dimensional}.
\end{proof}

\section{Dimension-free stability at the critical third moment}
\label{sec:third-moment}

At \(p=3\), the fourth-order convexity used above is no longer available.  We
work instead with the squared coefficients.  At each step, the largest and
smallest squared coefficients are replaced by their average.  This preserves
their sum, moves the vector toward the flat point, and decreases \(q\).
Conditioning on the remaining signs reduces the corresponding increase in the
third moment to a two-variable calculation; a small-ball estimate makes that
increase uniform in the residual sum.

\subsection{The simplex formulation}

Signs of the coefficients may be absorbed into the Rademacher variables.
Write
\[
x_i=a_i^2,\qquad x_i\ge0,\qquad \sum_{i=1}^n x_i=1,
\]
and, for $x$ in the probability simplex, set
\begin{equation}\label{eq:F-simplex}
F_n(x)\defeq
\E\abs{\sum_{i=1}^n\sqrt{x_i}\,\eps_i}^3,
\qquad
q(x)\defeq\sum_{i=1}^n x_i^2.
\end{equation}
At the flat point $u_n=(1/n,\ldots,1/n)$,
\[
F_n(u_n)=\E\abs{\overline S_n}^3,
\]
while
\begin{equation}\label{eq:simplex-distance}
\sum_{i=1}^n\left(x_i-\frac1n\right)^2=q(x)-\frac1n.
\end{equation}
The proof will compare the increase of \(F_n\) with the decrease of \(q\) at
each averaging step.

We use the following small-ball estimate.

\begin{lemma}[Dzindzalieta--G\"otze \cite{DzindzalietaGotze}]
\label{lem:small-ball}
Let
\[
R=\sum_{k=1}^m b_k\eps_k,
\qquad
\sum_{k=1}^m b_k^2\le1,
\]
and suppose that $\max_k\abs{b_k}\le\tau\le1$.  Then
\begin{equation}\label{eq:small-ball}
\mathbb P\{\abs{R}\le\tau\}\ge\frac3{16}\tau.
\end{equation}
\end{lemma}

\subsection{A two-point smoothing estimate}

Put
\begin{equation}\label{eq:kappa-definition}
\kappa\defeq\frac{5\sqrt2-7}{2}.
\end{equation}
The next lemma quantifies the effect of averaging two squared coefficients.
Its expectation is only over the two displayed Rademacher variables.

\begin{lemma}[Cubic two-point smoothing]\label{lem:cubic-smoothing}
Let $x\ge y\ge0$, $A=x+y>0$, and $r\in\R$.  If $\eps,\eps'$ are independent Rademacher
variables, then
\begin{align}
\Gamma(r;x,y)\defeq{}&
\E_{\eps,\eps'}\abs{r+\sqrt{A/2}\,\eps+\sqrt{A/2}\,\eps'}^3 \notag\\
&-\E_{\eps,\eps'}\abs{r+\sqrt x\,\eps+\sqrt y\,\eps'}^3\ge0.
\label{eq:cubic-smoothing-nonnegative}
\end{align}
Moreover, when $\abs r\le\sqrt A$,
\begin{equation}\label{eq:cubic-smoothing-quantitative}
\Gamma(r;x,y)\ge
\kappa\,\frac{(x-y)^2}{\sqrt A}.
\end{equation}
\end{lemma}

\begin{proof}
By homogeneity, normalize \(A=1\); symmetry allows us to replace \(r\) by
\(s=\abs r\).  The two possible magnitudes of
\(\sqrt{x}\,\eps+\sqrt{y}\,\eps'\) are described by
\[
u=\sqrt x+\sqrt y,\qquad v=\sqrt x-\sqrt y.
\]
Then $1\le u\le\sqrt2$, $0\le v\le1$, $u^2+v^2=2$, and $uv=x-y$.
For \(s,t\ge0\), write
\[
H(s,t)=\frac{\abs{s+t}^3+\abs{s-t}^3}{2}.
\]
Averaging first over the common sign of each magnitude gives
\begin{equation}\label{eq:cubic-profile-main}
\Gamma(s;x,y)
=\frac12H(s,\sqrt2)+\frac12s^3
 -\frac12H(s,u)-\frac12H(s,v).
\end{equation}

Since \(H(s,t)=t^3+3ts^2\) for \(s\le t\) and
\(H(s,t)=s^3+3st^2\) for \(s\ge t\), differentiating separately on
\([0,v]\), \([v,u]\), and \([u,1]\) shows that the right-hand side of
\Cref{eq:cubic-profile-main} is nonincreasing on \([0,1]\).  The derivatives
are recorded in \Cref{app:cubic-smoothing-calculation}.  Hence
\[
\Gamma(s;x,y)\ge\Gamma(1;x,y).
\]
At \(s=1\), substitution of the appropriate branches gives
\[
\Gamma(1;x,y)
=\frac{5\sqrt2}{2}-3-\frac{u^3}{2}
  +\frac{3u^2}{2}-\frac{3u}{2}.
\]
Using $(x-y)^2=u^2(2-u^2)$, one obtains
\begin{equation}\label{eq:cubic-profile-factorization}
\Gamma(1;x,y)-\kappa(x-y)^2
=\kappa(u-1)^2(u-\sqrt2)(u-5-4\sqrt2)\ge0.
\end{equation}
This proves \Cref{eq:cubic-smoothing-quantitative} when \(A=1\).

For \(1\le s\le u\), the same formulas show that the expression in
\Cref{eq:cubic-profile-main} decreases to
\(\frac12(\sqrt2-u)^3\).  It then equals
\(\frac12(\sqrt2-s)^3\) on \([u,\sqrt2]\) and vanishes for
\(s\ge\sqrt2\).  It is therefore nonnegative for every \(s\ge1\), proving
\Cref{eq:cubic-smoothing-nonnegative}.  Rescaling from \(A=1\) completes the
proof.
\end{proof}

The conditional estimate can now be averaged over the remaining coordinates.

\begin{lemma}[Extreme-pair gain]\label{lem:extreme-pair}
Let $x\ge y\ge0$, put $A=x+y>0$, and let
\[
R=\sum_k b_k\eps_k
\]
be independent of two additional Rademacher variables $\eps,\eps'$.  Assume
\begin{equation}\label{eq:residual-assumptions}
\sum_k b_k^2\le1,
\qquad
\max_k\abs{b_k}\le\sqrt A,
\qquad
A\le1.
\end{equation}
Then
\begin{align}
&\E\abs{R+\sqrt{A/2}\,\eps+\sqrt{A/2}\,\eps'}^3
-\E\abs{R+\sqrt x\,\eps+\sqrt y\,\eps'}^3 \notag\\
&\hspace{35mm}\ge\frac{3\kappa}{16}(x-y)^2.
\label{eq:extreme-pair-gain}
\end{align}
\end{lemma}

\begin{proof}
Condition on $R$.  By \Cref{lem:cubic-smoothing}, the conditional gain is always
nonnegative, and on the event $\{\abs R\le\sqrt A\}$ it is at least
\[
\kappa\frac{(x-y)^2}{\sqrt A}.
\]
If necessary, pad the residual sum with a zero coefficient; this does not
change its law or any assumption.  Since \(A\le1\), we may apply
\Cref{lem:small-ball} with \(\tau=\sqrt A\), obtaining
\[
\mathbb P\{\abs R\le\sqrt A\}\ge\frac3{16}\sqrt A.
\]
Averaging the conditional estimate proves \Cref{eq:extreme-pair-gain}.
\end{proof}

\subsection{The dimension-three case}

\begin{lemma}\label{lem:dim-three}
In dimension three,
\[
\E\abs{\sum_{i=1}^3 a_i\eps_i}^3
\le \E\abs{\overline S_3}^3
-(5\sqrt3-6\sqrt2)\Delta_3(a).
\]
The corresponding quotient is attained by
$(a_1^2,a_2^2,a_3^2)=(1/2,1/2,0)$, up to signs and permutations.
\end{lemma}

\begin{proof}
Set \(C_\star=5\sqrt3-6\sqrt2\).  After changing signs and permuting
coordinates, assume \(a\ge b\ge c\ge0\) and
\(a^2+b^2+c^2=1\).  The identities
\[
\E\abs{\overline S_3}^3=\frac5{2\sqrt3},
\qquad
\sqrt2+\frac{C_\star}{2}
=\frac5{2\sqrt3}+\frac{C_\star}{3},
\]
show that the desired inequality is equivalent to
\begin{equation}\label{eq:dim-three-penalized}
\E\abs{a\eps_1+b\eps_2+c\eps_3}^3
+C_\star(a^4+b^4+c^4)
\le \sqrt2+\frac{C_\star}{2}.
\end{equation}

The cases \(a\ge b+c\) and \(a\le b+c\) exhaust the ordered region.  Suppose
first that \(a\ge b+c\).  Conditioning on the last two signs gives
\[
\E\abs{a\eps_1+b\eps_2+c\eps_3}^3=a(3-2a^2).
\]
Also $b^4+c^4\le(b^2+c^2)^2$.  With $t=a^2\ge1/2$, the left-hand side of
\Cref{eq:dim-three-penalized} is therefore at most
\[
\sqrt t\,(3-2t)+C_\star\bigl(t^2+(1-t)^2\bigr).
\]
Its derivative is
\[
(2t-1)\left(2C_\star-\frac3{2\sqrt t}\right),
\]
which is nonpositive on \([1/2,1]\).  Thus the maximum in this region occurs at
\(t=1/2\), proving \Cref{eq:dim-three-penalized} there, with equality only at
\((a,b,c)=(1/\sqrt2,1/\sqrt2,0)\).

Now suppose \(a\le b+c\), and put \(s=a+b+c\) and \(r=abc\).  Expanding the
four sign patterns in this triangle region gives
\[
\E\abs{a\eps_1+b\eps_2+c\eps_3}^3=\frac{s^3-12r}{2},
\qquad
a^4+b^4+c^4=1-\frac12(s^2-1)^2+4sr.
\]
For fixed \(s\), the left-hand side of
\Cref{eq:dim-three-penalized} is affine in \(r\), with coefficient
\(4C_\star s-6<0\).  It is therefore largest when \(abc\) is smallest.
Under \(a\ge b\ge c\ge0\), \(a+b+c=s\), and
\(a^2+b^2+c^2=1\), that minimum occurs when the two largest coordinates
coincide; see \Cref{app:dim-three-calculus}.  It remains to consider
\[
a=b=\frac1{\sqrt{2+z^2}},\qquad
c=\frac z{\sqrt{2+z^2}},\qquad 0\le z\le1.
\]
For this family \Cref{eq:dim-three-penalized} becomes
\begin{equation}\label{eq:dim-three-one-variable}
\frac{4+3z^2+z^3/2}{(2+z^2)^{3/2}}
+C_\star\frac{2+z^4}{(2+z^2)^2}
\le \sqrt2+\frac{C_\star}{2}.
\end{equation}
The derivative calculation in \Cref{app:dim-three-calculus} shows that the
left-hand side first decreases and then increases, so it has no interior
maximum.  Its endpoint values are equal, which proves
\Cref{eq:dim-three-one-variable}.  The endpoints correspond respectively to
\((1/\sqrt2,1/\sqrt2,0)\) and the flat vector.  The non-flat endpoint also
gives
\[
C_3^{\mathrm{opt}}\le5\sqrt3-6\sqrt2.
\]
\end{proof}

\subsection{Iteration to the flat vector}

\begin{proof}[Proof of \Cref{thm:third-moment}]
Start from \(x^{(0)}=(a_1^2,\ldots,a_n^2)\).  If it is flat, there is nothing
to prove.  Otherwise, let \(M_k\) and \(m_k\) be the largest and smallest
coordinates of \(x^{(k)}\), and obtain \(x^{(k+1)}\) by replacing this pair
with their average.  This operation keeps the vector in the simplex.  If the
flat point is reached after finitely many steps, keep the sequence constant
thereafter.

The residual Rademacher sum has squared coefficient sum at most one, and every
one of its coefficients has magnitude at most
\(\sqrt{M_k}\le\sqrt{M_k+m_k}\).  Thus
\Cref{lem:extreme-pair} gives
\begin{equation}\label{eq:iteration-F}
F_n(x^{(k+1)})-F_n(x^{(k)})
\ge\frac{3\kappa}{16}(M_k-m_k)^2.
\end{equation}
The same averaging step decreases \(q\) by
\begin{equation}\label{eq:iteration-q}
q(x^{(k)})-q(x^{(k+1)})=\frac12(M_k-m_k)^2.
\end{equation}
Eliminating \((M_k-m_k)^2\) between the two estimates gives
\begin{equation}\label{eq:iteration-telescope}
F_n(x^{(k+1)})-F_n(x^{(k)})
\ge\frac{3(5\sqrt2-7)}{16}
\bigl(q(x^{(k)})-q(x^{(k+1)})\bigr).
\end{equation}

Since \(q\ge1/n\) on the simplex, summing \Cref{eq:iteration-q} shows that
\(\sum_k(M_k-m_k)^2<\infty\), and hence \(M_k-m_k\to0\).  The mean coordinate
remains \(1/n\), so
\[
\max_i\left|x_i^{(k)}-\frac1n\right|\le M_k-m_k\longrightarrow0;
\]
therefore \(x^{(k)}\to u_n\).  The function \(F_n\) is a finite average of
continuous functions of \(\sqrt{x_i}\), and is continuous on the simplex.
Summing \Cref{eq:iteration-telescope} through step \(K\) and then letting
\(K\to\infty\) gives
\[
F_n(u_n)-F_n(x^{(0)})
\ge\frac{3(5\sqrt2-7)}{16}
\left(q(x^{(0)})-\frac1n\right) \geq \frac{1}{100}
\left(q(x^{(0)})-\frac1n\right).
\]
Together with \Cref{eq:simplex-distance}, this proves \Cref{eq:third-moment-main}.
The upper bound in \Cref{eq:C3-opt-bounds} is supplied by \Cref{lem:dim-three}, completing
the theorem.
\end{proof}

\begin{remark}\label{rem:C3-sharpness}
The proof establishes dimension-free stability but does not determine the exact value of
$C_3^{\mathrm{opt}}$.  The dimension-three vector above is a natural sharpness candidate.  Any
improvement of the constant $3/16$ in \Cref{lem:small-ball} immediately improves the lower
bound obtained by this method.
\end{remark}

\begin{conjecture}\label{conjecture}
We conjecture that, for every $n\ge1$ and every $a\in\R^n$ with $\sum_i a_i^2=1$,
\begin{equation}\label{eq:conjecture}
\E\abs{\sum_{i=1}^n a_i\eps_i}^3
\le \E\abs{\overline S_n}^3
-(5\sqrt3-6\sqrt2)\,\Delta_n(a).
\end{equation}
\end{conjecture}

\section{Concluding perspective}

The fourth-power mass \(q\) exposes two complementary extremal geometries.  At
fixed \(q\), the upper law consists of one Bernoulli spike and a Gaussian
cloud; comparison with its endpoint distributions yields Gaussian stability.
At fixed dimension, the fourth-order principle leaves one possible spike among
otherwise equal coefficients, and the cubic-quotient lemma shows that the
flat choice is optimal.  The interval \(4<p<5\) lies below the convexity
threshold of the fixed-moment principle and instead requires the direct
coefficientwise argument.  At \(p=3\), the exact dimension-free quadratic
stability constant remains open, as stated in \Cref{rem:C3-sharpness,conjecture}.
For \(0<q<1\), the fixed-\(q\) envelope is a supremum in the closure of
finite Rademacher sums rather than an attained finite-dimensional law.  The
results therefore distinguish the proved sharp envelopes and
finite-dimensional optimizers from the remaining question of the exact
third-moment constant.

\appendix
\section{Calculations for the Gaussian-shift chord}
\label{app:gaussian-chord-calculations}

This appendix supplies the curvature and comparison computations used in
\Cref{lem:gaussian-chord}.

\subsection{The curvature calculation}\label{app:gaussian-chord-curvature}

For \(\beta>0\), write
\[
M_\beta(t)=\E\abs{G+t}^{\beta}.
\]
Gaussian integration by parts gives
\begin{equation}\label{eq:shift-moment-recurrence}
M_{\beta+2}(t)=\bigl(\beta+1+t^2\bigr)M_\beta(t)+tM_\beta'(t).
\end{equation}
Let \(\alpha=p-4\) and \(m=M_\alpha\).  Applying
\Cref{eq:shift-moment-recurrence} twice, and using \Cref{eq:gaussian-ode}, yields
\begin{align*}
M_{\alpha+2}(t)
&=\bigl(\alpha+1+t^2\bigr)m(t)+tm'(t),\\
M_{\alpha+4}(t)
&=\Bigl((\alpha+1)(\alpha+3)+3(\alpha+2)t^2+t^4\Bigr)m(t)\\
&\qquad+t\bigl(2\alpha+5+t^2\bigr)m'(t).
\end{align*}
Now \(h(u)=v^{\alpha+4}M_{\alpha+4}(t)\), where
\(v=\sqrt{1-u^2}\) and \(t=u/v\).  Since
\[
\frac{\dd v}{\dd u}=-t,
\qquad
\frac{\dd t}{\dd u}=v^{-3},
\]
two differentiations and one further use of \Cref{eq:gaussian-ode} give
\Cref{eq:curvature-before-Q}.

\subsection{The Riccati comparison}\label{app:gaussian-chord-comparison}

With
\[
R(t)=\frac{m'(t)}{tm(t)},
\qquad
Q(t)=\alpha\frac{3+(\alpha+2)t^2}{3+(2\alpha+3)t^2},
\]
\Cref{eq:gaussian-ode} gives \Cref{eq:R-riccati}.  For \(Q\), bringing the
two sides of \Cref{eq:Q-supersolution} to a common denominator gives
\begin{equation}\label{eq:Q-error}
\begin{split}
&tQ'(t)-\bigl[\alpha-(1+t^2)Q(t)-t^2Q(t)^2\bigr]\\
&\qquad=
\frac{\alpha(\alpha+1)(\alpha+3)t^4\bigl(4+(\alpha+2)t^2\bigr)}
{\bigl(3+(2\alpha+3)t^2\bigr)^2}>0.
\end{split}
\end{equation}
This proves \Cref{eq:Q-supersolution}.

To justify the comparison at the singular endpoint \(t=0\), set \(P=Q-R\)
and denote the positive right-hand side of \Cref{eq:Q-error} by \(E(t)\).
Subtracting the two Riccati equations gives
\[
tP'(t)+\bigl(1+t^2+t^2(Q(t)+R(t))\bigr)P(t)=E(t).
\]
The integrating factor for this equation is
\[
I(t)=t\exp\!\left(\int_0^t r\bigl(1+Q(r)+R(r)\bigr)\dd r\right).
\]
Since \(E(r)=O(r^4)\) and \(I(r)=O(r)\) near the origin, the integrand below
is integrable there.  The condition \(P(t)\to0\) as \(t\downarrow0\) and
variation of constants give
\[
P(t)=\frac1{I(t)}\int_0^t \frac{I(r)}rE(r)\dd r.
\]
The integrand is positive, and therefore \(P(t)>0\) for every \(t>0\), as used in the
main proof.

\section{Calculations for the third-moment argument}
\label{app:third-moment-calculations}

\subsection{The two-point profile}\label{app:cubic-smoothing-calculation}

We verify the monotonicity used in \Cref{lem:cubic-smoothing}.  Use the
normalized notation from its proof, so \(x+y=1\):
\[
u=\sqrt x+\sqrt y,\qquad v=\sqrt x-\sqrt y,
\qquad u^2+v^2=2.
\]
Let the right-hand side of \Cref{eq:cubic-profile-main} be denoted by
\(\mathcal G(s)\).  From the two branches of \(H\), for \(0\le s\le v\),
\begin{equation}\label{eq:third-app-first-derivative}
\mathcal G'(s)=\frac32s\bigl[s+2(\sqrt2-u-v)\bigr].
\end{equation}
Now
\[
2u+v=2\sqrt{2-v^2}+v\ge2\sqrt2,
\]
because the left-hand side is concave on $[0,1]$ and its endpoint values are at least
$2\sqrt2$.  Hence $v\le2(u+v-\sqrt2)$, and \Cref{eq:third-app-first-derivative} is
nonpositive.

For \(v\le s\le u\),
\begin{equation}\label{eq:third-app-second-derivative}
\mathcal G'(s)=3(\sqrt2-u)s-\frac32v^2
=3(\sqrt2-u)\left(s-\frac{\sqrt2+u}{2}\right)\le0,
\end{equation}
since \(s\le u\le(\sqrt2+u)/2\).  Thus \(\mathcal G\) is
nonincreasing on \([0,u]\); in particular it is
nonincreasing on $[0,1]$.

For $u\le s\le\sqrt2$, direct substitution gives
\[
\mathcal G(s)=\frac12(\sqrt2-s)^3,
\]
and for \(s\ge\sqrt2\) all terms lie in the second branch of \(H\), so
\(\mathcal G(s)=0\).  This proves
the remaining nonnegativity assertion used in the main proof.

\subsection{The dimension-three reduction}\label{app:dim-three-calculus}

We give the two calculations used in the proof of \Cref{lem:dim-three}.
Again write \(C_\star=5\sqrt3-6\sqrt2\).  Fix \(s=a+b+c\) and assume
\(a\ge b\ge c\ge0\), \(a^2+b^2+c^2=1\).  Since
\[
b+c=s-a,
\qquad
b^2+c^2=1-a^2,
\]
we have
\[
bc=a^2-sa+\frac{s^2-1}{2},
\qquad
abc=a^3-sa^2+\frac{s^2-1}{2}a.
\]
The ordered feasible interval for $a$ starts when $a=b$, at
\[
a_0=\frac{2s+\sqrt{6-2s^2}}6.
\]
Moreover,
\[
\frac{d}{da}(abc)
=3\left(a-\frac{2s-\sqrt{6-2s^2}}6\right)
  \left(a-\frac{2s+\sqrt{6-2s^2}}6\right)\ge0
\]
on that interval.  Thus $abc$ is minimized at $a=b$, as claimed.

For \Cref{eq:dim-three-one-variable}, denote its left-hand side by \(J(z)\).
Differentiation yields
\[
J'(z)=\frac{z(1-z)}{(z^2+2)^3}
\left(3z\sqrt{z^2+2}-8C_\star(z+1)\right).
\]
The factor in parentheses is strictly increasing on $[0,1]$, since its derivative is
\[
\frac{6(z^2+1)}{\sqrt{z^2+2}}-8C_\star
\ge3\sqrt2-8C_\star>0.
\]
It is negative at $z=0$ and positive at $z=1$.  Consequently $J$ first decreases and then
increases, and therefore has no interior maximum.  Finally,
\[
J(0)=J(1)=\sqrt2+\frac{C_\star}{2},
\]
which proves \Cref{eq:dim-three-one-variable}.

\section{The scalar estimates}\label{app:scalar}

The appendix records the elementary estimates used in \Cref{lem:small-q-scalar,lem:large-q-scalar}.
No numerical optimization is involved.

\subsection{The estimate for small fourth-power mass}\label{app:scalar-small}

\begin{proof}[Proof of \Cref{lem:small-q-scalar}]
At the point \(x_0=2^{-1/2}\), put
\[
\rho=1-\frac{2x_0}{3}=1-\frac{\sqrt2}{3}.
\]
The difference between the two sides of \Cref{eq:small-q-scalar} is
\[
D(p)\defeq
\frac{1-\rho^{(p-2)/2}}{x_0}-1+\frac1{\mu_p}.
\]
Thus the desired inequality is \(D(p)\ge0\).  Substitution of \(\mu_4=3\)
gives \(D(4)=0\), so it remains to prove that \(D\) is strictly increasing.

For this purpose, recall the logarithmic derivative of the Gaussian moment:
\[
\ell(p)\defeq\frac{\dd}{\dd p}\log\mu_p
=\frac12\log2+\frac12\psi\left(\frac{p+1}{2}\right),
\]
where \(\psi\) is the digamma function.  Differentiation gives
\begin{equation}\label{eq:Dprime}
D'(p)
=-\frac{\rho^{(p-2)/2}\log\rho}{2x_0}
-\frac{\ell(p)}{\mu_p}.
\end{equation}
Both \(-\rho^{(p-2)/2}\log\rho/(2x_0)\) and
\(\ell(p)/\mu_p\) are positive.  Hence \(D'(p)>0\) is equivalent to the
following ratio being greater than one:
\[
\mathcal R(p)=
\frac{-\rho^{(p-2)/2}\log\rho/(2x_0)}{\ell(p)/\mu_p}.
\]
The function \(\ell\) is increasing and \(\ell'\) is decreasing; this follows from the
positivity of the trigamma function and the negativity of its derivative.  The special values at
\(p=4\) satisfy
\begin{equation}\label{eq:ell-special}
\ell(4)=\frac43-\frac{\gamma_E+\log2}{2}>\frac23,
\qquad
\ell'(4)=\frac{\pi^2}{8}-\frac{10}{9}<\frac18,
\end{equation}
and \(\rho>e^{-2/3}\); these elementary bounds are verified in
\Cref{lem:elementary-constants}.  Consequently, for \(p\ge4\),
\[
\frac{\dd}{\dd p}\log\mathcal R(p)
=\frac12\log\rho+\ell(p)-\frac{\ell'(p)}{\ell(p)}
>-\frac13+\frac23-\frac3{16}=\frac7{48}>0.
\]
Thus \(\mathcal R\) is strictly increasing.  It remains to check that
\(\mathcal R(4)>1\).  From \(\gamma_E+\log2>19/15\),
\[
\frac{\ell(4)}{\mu_4}=\frac{\ell(4)}3<\frac7{30}.
\]
Writing \(y=\sqrt2/3\), so \(\rho=1-y\), the first four terms of
\(-\log(1-y)=\sum_{k\ge1}y^k/k\) give
\[
-\frac{\rho\log\rho}{\sqrt2}
>\frac{77-14\sqrt2}{243}>\frac7{30},
\]
where the last inequality follows at once from \(\sqrt2<10/7\).
These two estimates show that
\(\mathcal R(4)>1\), and monotonicity gives \(\mathcal R(p)>1\) for every \(p\ge4\).
By \Cref{eq:Dprime}, this means \(D'(p)>0\).  Together with
\(D(4)=0\), this proves the lemma and its strict form for \(p>4\).
\end{proof}

\subsection{The estimate for large fourth-power mass}\label{app:scalar-large}

\begin{proof}[Proof of \Cref{lem:large-q-scalar}]
First, \Cref{eq:ell-special} and the estimate
\(\gamma_E+3\log2<8/3\) in \Cref{lem:elementary-constants} give
\(\ell(4)>\log2\).

We begin with the logarithmic estimate
\begin{equation}\label{eq:large-q-log}
3\log2\frac{1-q}{3-2q}
+\frac12\log\frac{3-2q}{15-30q+16q^{3/2}}
\ge0.
\end{equation}
For \Cref{eq:large-q-log}, put \(x=\sqrt q\) and, temporarily, set
\[
A=3-2x^2,
\qquad
B=15-30x^2+16x^3.
\]
These temporary symbols are \(A=A(q)\) and \(B=B(q)\) from
\Cref{eq:A-B}, expressed in terms of \(x\).  Since
\(B'(x)=12x(4x-5)<0\), we have \(B(x)\ge B(1)=1\), so the logarithm below is
well defined.  Define
\[
K(x)=3(\log2)\frac{1-x^2}{A}+\frac12\log\frac{A}{B},
\qquad 2^{-1/2}\le x\le1.
\]
Let
\[
N(x)=2A (15-18x+4x^3)-3(\log2)B.
\]
Differentiation gives
\begin{equation}\label{eq:Kprime}
K'(x)=\frac{2xN(x)}{A^2B}
\end{equation}
and
\[
N''(x)=(-120+180\log2)+(576-288\log2)x-320x^3.
\]
On the interval in question, \(N'''(x)<0\), so \(N''\) is decreasing.
Since \(x\le1\) and \(N''(1)=136-108\log2>0\), it follows that
\(N''(x)>0\) throughout the interval.  Thus \(N\) is strictly convex.
Moreover,
\[
N(1)=2-3\log2<0,
\]
and
\[
N(2^{-1/2})=60-32\sqrt2-12\sqrt2\log2>0;
\]
for the latter it is enough to use \(\log2<7/10\) and \(\sqrt2<10/7\).  Because \(N'\) is
increasing, \(N\) decreases up to at most one minimum and then increases.  The endpoint signs
force exactly one zero, on the decreasing branch.  By \Cref{eq:Kprime}, \(K\) first increases
and then decreases.  Since
\[
K(2^{-1/2})=K(1)=0,
\]
we conclude that \(K\ge0\), which is \Cref{eq:large-q-log}.

Fix \(q\in[1/2,1]\).  By \Cref{eq:Lambda}, \(\Lambda_p(q)\) is the
\(p\)-th moment of a nonnegative random variable equal to \(1\) with
probability \(q\) and to \(\abs G\) with probability \(1-q\).  Therefore
\(p\mapsto\log\Lambda_p(q)\) is convex.  Since
\(\Lambda_4(q)=A(q)\) and
\(\partial_p\Lambda_p(q)|_{p=4}=3(1-q)\ell(4)\), its supporting tangent at
\(p=4\), followed by \Cref{eq:large-q-log}, gives
\begin{align*}
\log\Lambda_p(q)
&\ge \log A(q)+(p-4)\frac{3(1-q)\ell(4)}{A(q)}\\
&\ge \log A(q)+\frac{p-4}{2}\log\frac{B(q)}{A(q)}.
\end{align*}
Hence
\[
A(q)^{(6-p)/2}B(q)^{(p-4)/2}\le\Lambda_p(q),
\qquad p\ge4.
\]

\end{proof}

\subsection{Elementary constants}

Here \(\gamma_E\) is Euler's constant, and
\(H_n\defeq\sum_{k=1}^n k^{-1}\) denotes the \(n\)-th harmonic number.

\begin{lemma}\label{lem:elementary-constants}
The following estimates hold:
\begin{equation}\label{eq:elementary-constants}
\frac{83}{120}<\log2<\frac{1733}{2500}<\frac7{10},
\qquad
\frac{23}{40}<\gamma_E<\frac7{12}.
\end{equation}
Consequently,
\[
\frac{19}{15}<\gamma_E+\log2<\frac43,
\qquad
\gamma_E+3\log2<\frac83.
\]
Moreover,
\[
\frac{\pi^2}{8}-\frac{10}{9}<\frac18,
\qquad
1-\frac{\sqrt2}{3}>e^{-2/3}.
\]
\end{lemma}

\begin{proof}
The identity
\[
\log2=2\sum_{k=0}^\infty\frac{1}{(2k+1)3^{2k+1}}
\]
with three retained terms and a geometric bound on the tail gives the two stated bounds for
\(\log2\).  The standard Euler--Maclaurin inequalities
\[
H_n-\log n-\frac1{2n}<\gamma_E
<H_n-\log n-\frac1{2n}+\frac1{12n^2}
\]
give the upper bound at \(n=1\).  For the lower bound, write
\(\log7=3\log2+\log(7/8)\); the preceding bound for \(\log2\), together with the first
three terms of the power series for \(\log(1-z)\) at \(z=1/8\), gives
\(\log7<109/56\).  The lower Euler--Maclaurin bound at \(n=7\) then gives
\(\gamma_E>23/40\).  The three consequences displayed immediately follow.

Finally, \(\pi<22/7\) gives \(\pi^2/8-10/9<1/8\).  Also,
\(\sqrt2<10/7\) gives \(1-\sqrt2/3>11/21\), while the first four terms of the exponential
series give \(e^{2/3}>21/11\).  Hence \(1-\sqrt2/3>e^{-2/3}\).
\end{proof}

\begin{remark}
The split at \(q=1/2\) is deliberately unoptimized.  It is the point at which the two scalar
arguments have transparent endpoint checks and meet without numerical optimization.
\end{remark}

\end{document}